\documentclass[12pt, authoryear]{article}

\usepackage{graphicx}
\usepackage{epsfig}
\usepackage{amssymb}
\usepackage{amsmath, amsthm}
\usepackage{soul}
\usepackage{xcolor}
\usepackage{accents}

\usepackage{algorithm}
\usepackage{algorithmic}
\usepackage{enumitem}
\usepackage{multicol}
\usepackage{indentfirst}
\usepackage{breqn}
\usepackage{natbib}

\graphicspath{{figdir}}

\newcommand{\bb}{\mathbb}

\newcommand{\R}{\bb R}

\newcommand{\ind}{{\mathbf{1}\,}}

\newcommand{\cs}{\mathcal{C}}
\newcommand{\fs}{\mathcal{F}}
\newcommand{\xs}{\mathcal{X}}
\newcommand{\ys}{\mathcal{Y}}

\newcommand{\is}{\mathcal{I}}

\newcommand{\tx}{\tilde{x}}
\newcommand{\hx}{\hat{x}}
\newcommand{\txs}{\tilde{\xs}}
\newcommand{\hxs}{\hat{\xs}}

\newcommand{\probP}{\text{I\kern-0.15em P}}

\newtheorem{proposition}{Proposition}

\newcommand{\RN}[1]{%
  \textup{\uppercase\expandafter{\romannumeral#1}}%
}

\title{Neural Network Models for\\ 
Contextual Regression}

\author{Seksan Kiatsupaibul
\thanks{Department of Statistics, Chulalongkorn University, Bangkok 10330, Thailand}
\and
Pakawan Chansiripas\footnotemark[1]
}
 
\begin{document}

\maketitle


\begin{abstract}
We propose a neural network model for contextual regression in which the regression model depends on contextual features that determine the active submodel and an algorithm to fit the model. The proposed simple contextual neural network (SCtxtNN) separates context identification from context-specific regression, resulting in a structured and interpretable architecture with fewer parameters than a fully connected feed-forward network.
We show mathematically that the proposed architecture is sufficient to represent contextual linear regression models using only standard neural network components.
Numerical experiments are provided to support the theoretical result, showing that the proposed model achieves lower excess mean squared error and more stable performance than feed-forward neural networks with comparable numbers of parameters, while larger networks improve accuracy only at the cost of increased complexity. The results suggest that incorporating contextual structure can improve model efficiency while preserving interpretability.

\end{abstract}


\section{Introduction}\label{s:intro}

Contextual regression concerns a family of regression models that are applied to different contexts that belongs to a common problem.  
Technically, a contextual regression model is a regression model whose parameters are functions of the explanatory variables or features.  
We are interested in the setting where the contextual features, 
the subset of features that identifies the contexts, are known.  
In addition, we assume that there are a finite number of contexts, and the contextual features cannot readily be 
transformed into indicator variables that can completely identify all of the contexts.  
Under this setting, even when the sub-model given a context is a linear regression model, 
the overall model is non-linear.  We propose an interpretable neural network architecture for 
a contextual regression problem and a method to fit the model.

The applications of the contextual regression can be numerously found in personalized product and services.  
In personalized medicine, pulse oximetry provides a vivid example of how a contextual regression model can be applied.  
A pulse oximeter determines the amount of oxygen in blood by sending a beam of light through a person's skin, 
measuring the light absorption signal and converting the signal to the blood oxygen level via a regression model.  
A standard pulse oximeter employs a single regression model for all skin shades.  However, there is  
a growing body of evidence that the accuracy of the oxygen measurement is compromised when the device 
is applied to a darker skin \citep{Sjoding2020}.  A better design is to adopt a contextual regression model for 
the oximetry, taking the skin shade as the context.  Note that the shade of the skin can also be determined 
from signals extracted from the light beam.  These skin shade signals can be taken as the contextual features.

When the sub-model given a context is a linear regression model, 
a contextual regression that is a linear model can be constructed when the contextual features are 
indicator variables.  In that case, by including the cross terms between the indicator contextual features and the other 
features, the resulting linear regression model can be served as a contextual regression model.  
Model fitting and analysis can be easily performed following the standard linear modeling methodology.  
The problem becomes complicated when the contextual variables are continuous variables, which is the case 
we consider here.  Assuming there are a finite number of contexts, the model is required to be 
able to classify the observations into contexts based on the values of the contextual features and, subsequently, 
to be able to predict the values of the responses based on the values of other features.  
In this case, a context regression problem becomes a combination of a classification problem and a regression problem.

A neural network is a unified framework for regression and classification modeling that can accommodate 
both linear and non-linear models.  Therefore, it is an ideal choice for modeling a contextual regression problem, 
which can be regarded as a combination of classification and regression.  However, the fully connected neural network model, 
the so-called feed-forward network, is over-parameterized and is prone to over-fitting.  
We propose an interpretable neural network model, called the contextual neural network (CtxtNN), 
that contains considerably lower number of parameters than that of the feed-forward neural network.  
We will show that CtxtNN is sufficient for modeling the contextual linear regression problems.  
We also introduce a model fitting algorithm that improves the chance of attaining a global optimal set of parameters for 
the contextual neural network.

The organization of the paper is as follows.  In Section~\ref{s:method}, we mathematically express 
the contextual regression problem under study and define CtxtNN for modeling the problem.  
The proposed neural network model is shown to be sufficient in some important cases.  
In Section~\ref{s:numerical}, we compare the performances between the proposed algorithm 
and the standard stochastic gradient when applied to various cases of CtxtNN.
A conclusion is provided in Section~\ref{s:conclusion}

\section{Methodology}\label{s:method}

We first describe the contextual regression problem to be addressed. 
Let $\xs = \R^p, p\geq 2$ and $\ys = \R$ be the feature space and the output space, respectively.  
Contexts are defined by the contextual features that are the elements of 
the contextual space $\txs = \R^q$, $0 < q < p$, a projection of the feature space $\xs$.  
Complement to the contextual space $\txs$ is the regressor space $\hxs=\R^{p-q} = \R^r$ 
where $r=p-q$ and $\xs = \hxs\times\txs = \R^r\times\R^q$.
For a feature vector $x=[x_1,\ldots,x_p]^\top \in\xs$, 
the vector of the first $r$ and the last $q$ elements, $\hx=[x_1, \ldots, x_r]^\top\in\hxs$ and $\tx=[x_{r+1},\ldots, x_p]^\top \in\txs$, 
represent the corresponding regressor feature vector and contextual feature vector, respectively.   
Therefore, 
\[x=[\hx, \tx]^\top.\]
We assume that there is a finite (small) number of contexts, 
indexed by $i\in\{1,\ldots,c\}=\is$.  
The contextual space $\txs$ is partitioned into context regions $\{\cs^1, \ldots \cs^c\}$, 
and $x\in\xs$ is said to belong to context $i, i\in\is$ if 
the corresponding contextual feature vector $\tx$ is in $\cs^i$.
For a contextual regression model, we assume there are a family $\fs$ of functions $\{f^1, \ldots, f^c\}$, each 
$f^i:\hxs\to\ys$ where $i\in\is$, and a regression function $f$ that is defined by
\begin{equation}\label{eqn:regfn0}
f(x) = f^i(\hx), \quad\text{if } \tx\in\cs^i, i\in\is.
\end{equation}
We impose the following structure on the problem.  We assume that there is a score function $g:\txs\to\R$ that 
map each context region $\cs^i$ to a \emph{connected} interval $I^i\subseteq \R$, $i\in\is$ such that $\{I^1, \ldots, I^c\}$ 
form a partition of $\R$.  Therefore, the regression function can be as well defined by
\begin{equation}\label{eqn:regfn1}
f(x) = f^i(\hx), \quad\text{if } g(\tx)\in I^i, i\in\is.
\end{equation}

\subsection{Simple Contextual Neural Network}\label{s:sctxtnn}
A simple yet important sub-class of the contextual regression models is the simple contextual linear regression model.  
It will serve as a building block for other contextual regression models in subsequent sections.
The simple contextual linear regression is a contextual regression model where all $f^i$, $i\in\is$, are linear functions
and there is one contextual feature, i.e., $q=1$, $r=p-1$ and the score function is the identity function $g(x)=x$.  
That is the contextual space $\R$ is partitioned into $c$ connected intervals, $\{I^1, \ldots, I^c\}$ 
and $x=[x_1,\ldots, x_p]^\top = [\hx, x_p]^\top$ is said to be in context $j$ if $x_p\in I^j$.  
By convention, let $I^j, j=1,\ldots c$ be the left-closed and right-open intervals.
The regression function 
is of the form
\begin{equation}
f(x)=\beta_0^j + \beta^j\cdot \hx, \quad\text{if }x_p\in I^j,
\end{equation}
where $\beta^i = ([\beta_1, \ldots,\beta_{r}]^j)^\top \in\R^{r}$ is the coefficient vector , 
the term $\beta_0^j\in\R$ is the bias term (or the intercept), and $\beta^i\cdot \hx$ is the 
dot product between $\beta^j$ and $\hx$.

We propose the following neural network model, called the simple contextual neural network (SCtxtNN), 
for the simple contextual linear regression.  The model can be depicted by the network diagram in Figure~\ref{fig:sctxtnn}.
The architecture of SCtxtNN is designed to employ only typical neural network components 
without any customized components, i.e., it contains only the linear components, ReLU (rectified linear unit) and 
the perceptron components.

\begin{figure}[h!]
\begin{center}
\includegraphics[scale=0.5]{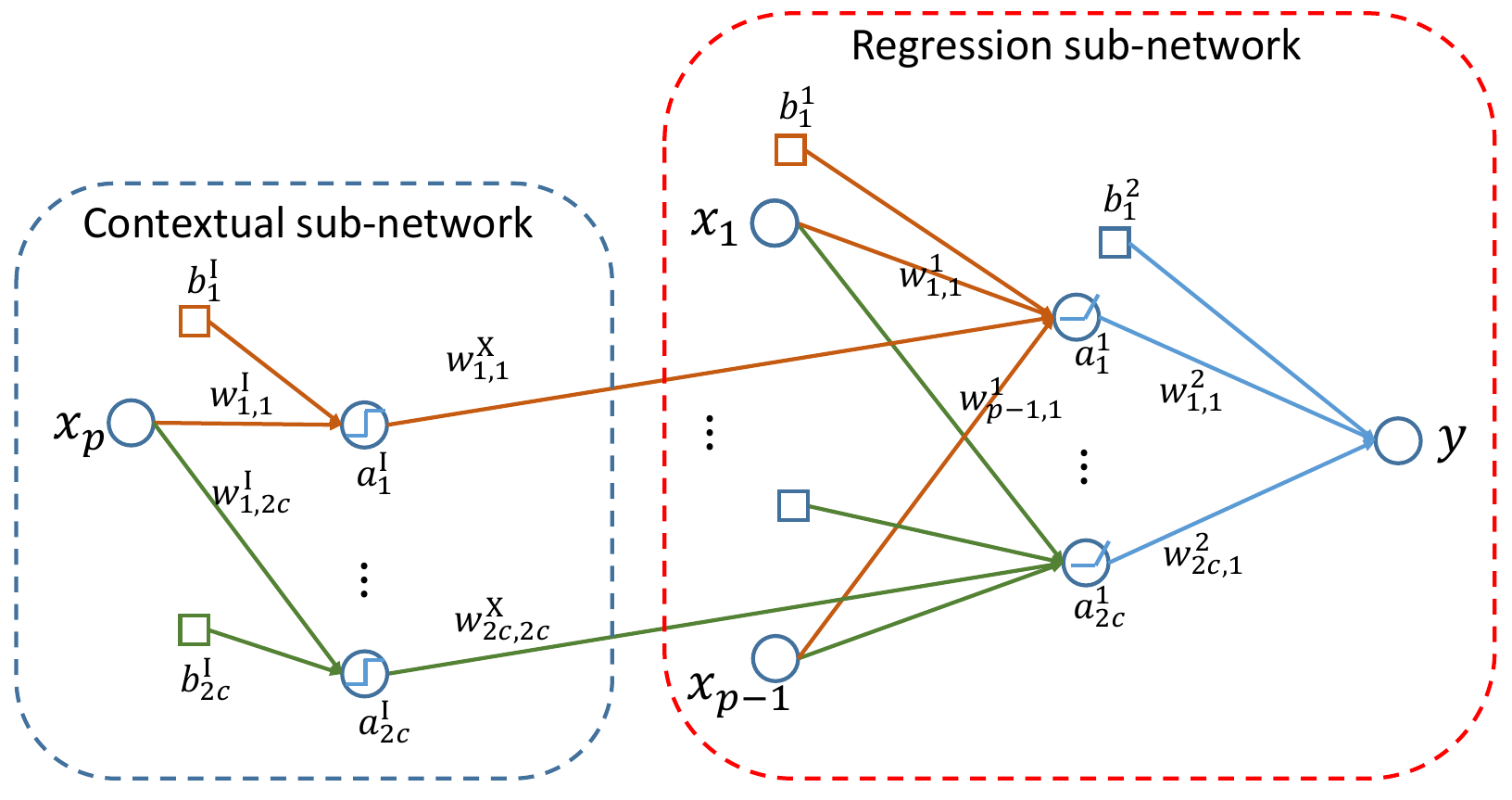}
\caption{Simple contextual neural network (SCtxtNN) architecture.}\label{fig:sctxtnn}
\end{center}
\end{figure}

For a neural network, we adopt the following notations.  
We are given $n$ training feature vectors, where 
$x$ is a specific feature vector, $y(x)$ is the desired output for example $x$, 
$\ell$ is the index of network layers of neurons or units, and 
$a^\ell(x)$ is the vector of activations output for input (feature or transformed feature) vector $x$.  
Let $L$ be the number of the network layers excluding the input layer.
\begin{enumerate}
  \item $w^\ell_{kj}$ is the {\em weight} of the $k^{\mbox {th}}$ {\em unit} in the $(\ell-1)^{\mbox {st}}$ layer applied to the $j^{\mbox {th}}$  unit in the $\ell^{\mbox {th}}$  layer,  $k=1,2,\ldots, m_{\ell-1}, j=1,2,\ldots, m_\ell, \ell=1,\ldots, L$.  
By convention, $l=0$ represents the input layer and $l=L$ represents the output layer.
  \item $b^\ell_j$ is the {\em bias} of the $j{\mbox {th}}$ unit in the $\ell{\mbox {th}}$ layer, $ j=1,2,\ldots, m_\ell, \ell=1,\ldots, L$.
  \item $a^\ell_j$ is the {\em activation} of the $j{\mbox {th}}$ unit in the $\ell{\mbox {th}}$ layer, $ j=1,2,\ldots, m_\ell, \ell=1,\ldots, L$.
\end{enumerate}

The network SCtxtNN consists of two sub-networks, the contextual sub-network and the regression sub-network.  
The contextual sub-network has only $L=1$ layer, i.e. the output layer indexed by $\RN{1}$. 
The input contains one contextual feature $x_p$.  
The output layer $\RN{1}$ contains $2c$ units, each unit is a \emph{perceptron} unit of the linear function 
of the contextual input features.  Therefore, there are $2c$ output $a^{\RN{1}}_j, j=1,\ldots, 2c$ 
from the output layer where 
\begin{equation}\label{eqn:perceptron}
a^{\RN{1}}_j = \tilde{a}(w_{1,j}^{\RN{1}} x_p + b^{\RN{1}}_j), \quad j=1,\ldots, 2c,
\end{equation}
and $\tilde{a}(z) = \ind_{\R^+}(z) $ is the the indicator function over the set of positive real numbers.

The regression sub-network has $L=2$ layers, which include one hidden layer $\ell=1$ and the output layer $\ell=2=L$.  
The input contains $r=p-1$ features, $x_i, i=1,\ldots,r$.
The hidden layer $\ell=1$ contains $2c$ units.  
For $j=1,\ldots, 2c$, the output of unit $j$ of layer $\ell=1$ 
is a \emph{ReLU} activation of the linear function of the inputs \emph{plus the output 
from the contextual subnetwork}, i.e.,
\begin{equation}
a^1_j = a\left(\sum_{k=1}^{r} w^1_{k,j}x_k + b^1_j + w^{\RN{10}}_{j,j}a^{\RN{1}}_j\right), \quad j=1,\ldots, 2c,
\end{equation}
where $a(z) = \max(0,z)$ is the ReLU activation function.
The output layer $l=2$ contains only one unit that represents the output.  This output is a linear function of 
the activation unit from the preceding layer, i.e.,
\begin{equation}
y = a^2_1 = \sum_{k=1}^{2c} w^2_{k,1}a^1_k + b^2_1.
\end{equation}

\begin{proposition}
Let $S\subset\hxs$ be a compact subset of the regressor space and let $T$ be a bounded interval in $\R$ such that 
$T\cap I_j \neq \emptyset$ for $j=1,\ldots,c$.  
For a simple contextual linear regression model, 
there exists SCtxtNN with output $y$ such that
\begin{equation}
y(x) = f(x)\quad\text{for all } x\in S\times T.
\end{equation}
\end{proposition}
\begin{proof}
By the definition of $T$ and $I^j$, there exist cut point $z_1 < z_2 < \ldots < z_c < z_{c+1}$ such that 
$x\in S\times T$ is in context $j$ if $z_{j} \leq x_p < z_{j+1}$, $j=1,\ldots,c$.    

Consider first the contextual sub-network.  
We set $a^{\RN{1}}_{2j-1}$ and $a^{\RN{1}}_{2j}$ to be the identifier for context $j$ for $j=1,\ldots,c$.  
That is we set $b^{\RN{1}}_{2j}$ and $w^{\RN{1}}_{1,2j}$ so that
\[w_{1,2j}^{\RN{1}} x_p + b^{\RN{1}}_{2j} > 0 \iff x_p < z_j.\]
and set $b^{\RN{1}}_{2j-1} = b^{\RN{1}}_{2j}$ and $w^{\RN{1}}_{1,2j-1} = w^{\RN{1}}_{1,2j}$.
With such parameter setting and that $a^{\RN{1}}_j$ is defined as a perceptron in (\ref{eqn:perceptron}), for $j=1,\ldots, c$,
\begin{equation}\label{eqn:ai}
a^{\RN{1}}_{2j-1} \text{ and } a^{\RN{1}}_{2j} = \left\{
								\begin{array}{ll}
								1, & \text{ if } x_p < z_j \\
								0, & \text{ if } x_p \geq z_j.
								\end{array}
								\right.
\end{equation}
We say that $a^{\RN{1}}_j$ is in on state or off state if it is zero or one, respectively.  
Therefore, when $x_p\in I^j$, i.e., $z_j \leq x_p <z_{j+1}$, 
we have $a^{\RN{1}}_{2k-1}$ and $a^{\RN{1}}_{2k}$ in on state for $k=1,\ldots,j$ and in off state for $k=j+1,\ldots,c$.  
We then set $w^{\RN{10}}_{j,j}$, for all $j=1,\ldots,2c$, to a large negative numbers so that, when 
$a^{\RN{1}}_j$ is in off state, it will turn off $a^1_j$ in the regression sub-network, i.e., making $a^1_j=0$.  
In fact, in what follows, we see that it suffices, for $j=1,\ldots,2c$, to set
\begin{equation}\label{eqn:wi}
w^{\RN{10}}_{j,j} = - 2\sup_{i\in\is} \left\{\sup_{\hat{x} \in S} \left|\beta^i\cdot \hat{x}\right| + 2\left|\beta^i_0\right|\right\}.
\end{equation}

Now construct the regression sub-network as follows.  Consider first layer $\ell=2$.  Set 
\begin{eqnarray}
b^2_1 & = & 0 \label{eqn:b20}\\
w^2_{2j-1,1} & = & 1 \label{eqn:w2l2j_1}\\
w^2_{2j,1} & = & -1 \label{eqn:w2l2j}
\end{eqnarray}
According to (\ref{eqn:ai}), (\ref{eqn:wi}), (\ref{eqn:b20}), (\ref{eqn:w2l2j_1}) and (\ref{eqn:w2l2j}), 
if $z_j \leq x_p \leq z_{j+1}$, then
\begin{equation}\label{eqn:y}
y = \sum_{k=1}^j \left(a^1_{2k-1} - a^1_{2k}\right).
\end{equation}

Lastly, we set the parameters for layer $\ell=1$ to give the required results. 
For the first unit of layer $\ell=1$, let
\begin{equation}
w^1_{k,1} = \beta^1_k, \quad\text{for } k=1,\ldots,r.
\end{equation}
The value of $b^1_1$ is set to carry the value of $\beta^1_0$ and to make 
$\sum_{k=1}^r w^1_{k,1}x_k + b^1_1$ be positive whenever $x\in S$.  It suffices to set
\begin{equation}
b^1_1= \beta^1_0 + \left|\beta^1_0\right| + \sup_{\hat{x}\in S}\left\{\left|\beta^1\cdot \hat{x}\right|\right\}.
\end{equation}
For the second unit of layer $\ell=1$, let
\begin{equation}
w_{k,2}^1 = 0, \quad\text{for } k=1,\ldots,r.
\end{equation}
The value of $b^1_2$ is set to offset that of $b^1_1$.
\begin{equation}
b^1_2 =  \left|\beta^1_0\right| + \sup_{\hat{x}\in S}\left\{\left|\beta^1\cdot \hat{x}\right|\right\} = b^1_1-\beta^1_0.
\end{equation}
Therefore, when $z_1 \leq x_p < z_2$, we have $a^{\RN{1}}_1=0$ and $a^{\RN{1}}_2=0$, so
\[
\sum_{k=1}^r w^1_{k,1}x_k + b^1_1 + w^{\RN{10}}_{1,1}a^{\RN{1}}_1 = \beta^1\cdot \hat{x} + b^1_1 > 0.
\]
Therefore,
\[
a^1_1 = \max(0, \beta^1\cdot \hat{x} + b^1_1) = \beta^1\cdot \hat{x} + b^1_1.
\]
Likewise,
\[
\sum_{k=1}^r w^1_{k,2}x_k + b^1_2 + w^{\RN{10}}_{2,2}a^{\RN{1}}_2 = b^1_2 
= b^1_1-\beta^1_0 > 0.
\]
Therefore,
\[
a^1_2 = \max(0, b^1_1-\beta^1_0 ) = b^1_1-\beta^1_0.
\]
Hence, 
\begin{equation}\label{eqn:a1a2}
a^1_1 - a^1_2  = \beta^1_0 + \beta^1\cdot\hat{x} = f^1(\hat{x}).
\end{equation}

For units $2j-1$ for $j=2,\ldots,c$ of layer $\ell=1$ of the regression sub-network, let 
\begin{equation}
w^1_{k,2j-1} = \beta^j_k - \beta^{j-1}_k, \quad\text{for } k=1,\ldots,r.
\end{equation}
The value of $b_{2j-1}^1$ for $j=2,\ldots,c$, is set to carry value of $\beta^j_0-\beta^{j-1}_0$ and make the linear sum positive.  
\begin{equation}
b^1_{2j-1}= (\beta^j_0 - \beta^{j-1}_0) + \left|\beta^j_0 - \beta^{j-1}_0\right| 
	+ \sup_{\hat{x}\in S}\left\{\left|(\beta^j-\beta^{j-1})\cdot \hat{x}\right|\right\}.
\end{equation}
For unit $2j$ of layer $\ell=2,\ldots,c$ of layer $\ell=1$ of the regression sub-network, let
\begin{equation}
w_{k,2j}^1 = 0, \quad\text{for } k=1,\ldots,r.
\end{equation}
The value of $b^1_{2j}$ is set to offset that of $b^1_{2j-1}$.
\begin{equation}
b^1_{2j} =  \left|\beta^j_0 - \beta^{j-1}_0\right| 
	+ \sup_{\hat{x}\in S}\left\{\left|(\beta^j-\beta^{j-1})\cdot \hat{x}\right|\right\} = b^1_{2j-1} -  (\beta^j_0 - \beta^{j-1}_0).
\end{equation}
Fix $j\in\{2,\ldots,c\}$.  When $z_j \leq x_p < z_{j+1}$, from (\ref{eqn:ai}), we have $a^{\RN{1}}_{2j-1}=0$ and $a^{\RN{1}}_{2j}=0$.  
Therefore,
\[
\sum_{k=1}^r w^1_{k,2j-1}x_k + b^1_{2j-1} + w^{\RN{10}}_{2j-1,2j-1}a^{\RN{1}}_{2j-1}  = \beta^j\cdot \hat{x} + b^1_{2j-1} > 0.
\]
Therefore,
\[
a^1_{2j-1} = \max(0, \beta^j\cdot \hat{x} + b^1_{2j-1}) = \beta^1\cdot \hat{x} + b^1_{2j-1}.
\]
Likewise,
\[
\sum_{k=1}^r w^1_{k,2j}x_k + b^1_{2j} + w^{\RN{10}}_{2j,2j}a^{\RN{1}}_2 = b^1_{2j} 
= b^1_{2j-1}-(\beta^j_0 - \beta^{j-1}_0) > 0.
\]
Therefore,
\[
a^1_{2j} = \max(0, b^1_{2j-1}-(\beta^j_0 - \beta^{j-1}_0) = b^1_{2j-1}-(\beta^j_0 - \beta^{j-1}_0).
\]
Hence, for $j=2,\ldots,c$,
\begin{equation}\label{eqn:aall}
a^1_{2j-1} + a^1_{2j}  = (\beta^j_0 - \beta^{j-1}_0) + (\beta^j - \beta^{j-1})\cdot\hat{x} = f^j(\hat{x})-f^{j-1}(\hat{x}).
\end{equation}

By (\ref{eqn:y}) and (\ref{eqn:a1a2}), when $z_1 \leq x_p < z_{2}$, we have, for $\hat{x}\in S$,
\begin{equation}
y = a^1_1 + a^1_2 = f^1(\hat{x}).
\end{equation}
By (\ref{eqn:y}) and (\ref{eqn:aall}), when $z_{2j-1} \leq x_p < z_{2j}$ for $j=2,\ldots,c$, we have, for $\hat{x}\in S$,
\begin{equation}
y = \sum_{k=1}^j \left(a^1_{2j-1} + a^1_{2j}\right) = f^1(\hat{x}) + \sum_{k=2}^j \left(f^k(\hat{x}) - f^{k-1}(\hat{x})\right) = f^j(\hat{x}).
\end{equation}

\end{proof}


\section{Numerical Study}\label{s:numerical}
In this section, we conduct a numerical experiment to compare the proposed simple contextual neural network (SCtxtNN) with standard feed-forward neural network models. We perform an experiment for the case one context feature and one regression feature, i.e., $p=2$ and $q=1$.  The objective of the experiment is to evaluate the predictive performance of the proposed architecture relative to conventional neural networks with comparable or larger numbers of parameters.

\subsection{Experimental Setup}

We compare three neural network models:
\begin{itemize}

\item \textbf{Simple Contextual Neural Network (SCtxtNN):}
described in Section~\ref{s:sctxtnn} with three contexts. In the implementation, sigmoid activation functions are used in the contextual subnetwork as a smooth approximation of the perceptron units, while ReLU activation is used in the regression subnetwork.

\item \textbf{Small Feed-Forward Neural Network (Small FF):}
A fully connected feed-forward neural network with two hidden layers, each having four hidden units with ReLU activation functions. The architecture is chosen so that the number of parameters is comparable to that of simple contextual neural network.

\item \textbf{Large Feed-Forward Neural Network (Large FF):}
A fully connected feed-forward neural network with two hidden layers, each having six hidden units with ReLU activation functions. The number of hidden units is chosen so that the proposed contextual neural network can be represented as a special case of this architecture, and therefore this model forms a superset of simple contextual neural network.

\end{itemize}

The total number of parameters of each model is shown in Table~\ref{tab:nparam}.

\begin{table}[h]
\centering
\begin{tabular}{l c}
\hline
Model & Number of parameters \\
\hline
Simple contextual neural network & 37 \\
Small feed-forward network & 37 \\
Large feed-forward network & 67 \\
\hline
\end{tabular}
\caption{Total number of parameters for each model}
\label{tab:nparam}
\end{table}

Data are generated according to the model in Section~\ref{s:sctxtnn} with three contexts. The contextual feature $\tilde{x}$ is generated from the uniform distribution on $[-1,1]$ and determines the context index $j \in \{1,2,3\}$ according to the intervals $[-1,-1/3)$, $[-1/3,1/3)$, and $[1/3,1]$. The regressor feature $\hat{x}$ is generated independently from the standard normal distribution.

For each context $j$, the regression coefficient vector $\beta_j$ is generated independently from the standard normal distribution. The response variable $y$ is then generated from the corresponding linear regression model with an additive Gaussian noise term 
$\varepsilon \sim N(0, 0.01^2)$.

For each simulation, a dataset of size $6000$ is generated. The dataset is then split into $1500$ training samples, $1500$ validation samples, and $3000$ test samples. Within each simulation, the same dataset is used for all models, while a new dataset is generated for each simulation.

All models are trained using the Adam optimizer with learning rate $0.001$ and mean squared error (MSE) loss. In each simulation, the three models are trained on the same training and validation sets to ensure a fair comparison. Each model is trained for $20{,}000$ epochs, and the training and validation losses are recorded at each epoch. After training, the mean squared error on the test set is computed for each model. The experiment is repeated for $50$ independent simulations.

\subsection{Results}

\begin{figure}[h!]
\begin{center}
\includegraphics[scale=0.6]{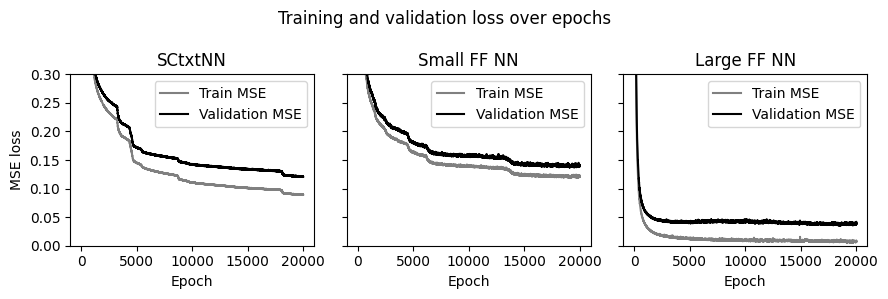}
\caption{Training and validation MSE over epochs for SCtxtNN, Small FF, and Large FF models.}
\label{fig:losscurve}
\end{center}
\end{figure}

Figure \ref{fig:losscurve} presents the training and validation loss across epochs for the three models, with each curve representing the mean of 50 simulations. All models demonstrate a consistent decrease in loss followed by a plateau, indicating successful convergence of the optimization process.

The small feed-forward neural network model shows the smallest gap between training and validation loss, which suggests limited model complexity, but its validation loss remains higher than that of the other models. 
The simple contextual neural network also converges smoothly, with a slightly larger gap between training and validation curves, while achieving lower validation loss than the small feed-forward network. 
The large feed-forward neural network model reaches very small training loss while the validation loss stabilizes. 
Overall, the loss curves mainly reflect differences in model complexity due to the different numbers of parameters and architectures defined in the experimental setup, but they do not by themselves determine the final predictive performance.

\begin{figure}[h!]
\begin{center}
\includegraphics[scale=0.8]{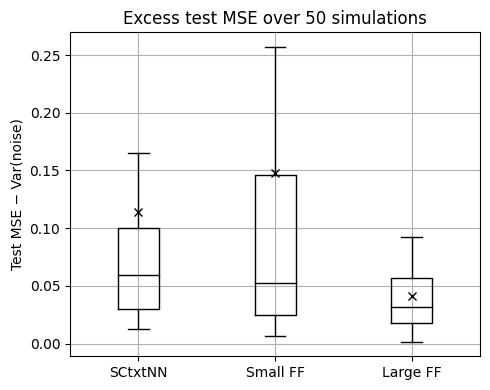}
\caption{Excess test MSE over 50 simulations for SCtxtNN, Small FF, and Large FF. Excess MSE is defined as test MSE minus the noise variance, so that zero corresponds to the optimal achievable error.}
\label{fig:boxplot}
\end{center}
\end{figure}

To compare predictive performance, we evaluate the excess test MSE shown in Figure~\ref{fig:boxplot}. 
The excess MSE is defined as the test mean squared error minus the variance of the noise, so that it measures the estimation error above the optimal predictor. 
With this definition, smaller values indicate better approximation of the true regression function.

The results show that the simple contextual neural network model achieves lower excess MSE than the small feed-forward neural network model. 
While the medians of the two models are similar, the contextual model has a lower mean excess MSE (marked by $\times$ in the figure), suggesting fewer large deviations across simulations.
In addition, its distribution exhibits lower variability, as indicated by the narrower box in the figure, which implies more stable performance. 
Since these two models have comparable numbers of parameters, this difference suggests that the context-based architecture is better aligned with the model structure described in Section~\ref{s:sctxtnn} than a generic fully connected network with similar capacity.

The large feed-forward neural network model achieves the lowest excess MSE overall. 
This behavior is expected because the larger network has substantially higher complexity and can approximate the target function more closely. 
However, this improvement is obtained by increasing the number of parameters rather than by using a model structure derived from the data-generating mechanism.

\section{Conclusion}\label{s:conclusion}

We compared a simple contextual neural network (SCtxtNN) with standard fully connected networks of different sizes. 
The results show that the structured model consistently outperforms a fully connected network with comparable numbers of parameters, while a much larger network can achieve lower error at the cost of substantially increased complexity.

These findings suggest that incorporating contextual structure into the model can improve predictive accuracy while keeping the network simple. 
By separating context identification from context-specific regression, the proposed architecture reflects the form of the data-generating process and can represent the relationship more efficiently than a generic fully connected network with a similar number of parameters.

In addition, the contextual model preserves interpretability, since the components of the network correspond to meaningful parts of the data-generating process, such as context identification and context-specific regression. 
This indicates that using contextual structure can improve model performance while maintaining a simple and interpretable representation.

\bibliographystyle{plainnat}
\bibliography{nnctxtref} 

\end{document}